\pdfoutput=1
\documentclass[11pt]{article}

\usepackage[margin=1in]{geometry}
\usepackage{times}
\usepackage{latexsym}
\usepackage[T1]{fontenc}
\usepackage[utf8]{inputenc}
\usepackage{microtype}
\usepackage{inconsolata}
\usepackage{silence}
\usepackage{amsmath,amssymb,amsthm,mathtools}
\usepackage[authoryear,round]{natbib}
\usepackage[colorlinks=true,linkcolor=blue,citecolor=blue,urlcolor=blue]{hyperref}

\newtheorem{proposition}{Proposition}
\newtheorem{remark}{Remark}
\newtheorem{theorem}{Theorem}

\newcommand{\V}{\mathcal{V}}
\newcommand{\E}{\mathbb{E}}
\newcommand{\KL}{\mathrm{KL}}
\newcommand{\JS}{\mathrm{JS}}

\newcommand{\TC}{\mathrm{TC}}
\newcommand{\Comm}{\mathrm{Comm}}
\newcommand{\Curl}{\mathcal{C}}

\title{Path-Dependent Denoising: A Non-Conservative Field Perspective on Order Collapse in Diffusion Language Models}
\author{Jeonseong Kim}
\date{}

\begin{document}
\maketitle

\begin{abstract}
Diffusion language models (DLMs) offer a structural alternative to autoregressive generation: denoising can update tokens in arbitrary orders or in parallel rather than along a fixed left-to-right chain. In practice, fast DLM decoding remains strongly order-sensitive and often drifts toward autoregressive-like trajectories. We trace this tension to compatibility. At each reverse-time step, a DLM provides local denoising conditionals over the unresolved tokens. Arbitrary-order denoising becomes well defined when these local conditionals compose into order-invariant pseudo-joints.

We formalize this view by defining order-induced pseudo-joints and a local denoising circulation: the log-ratio between the two pseudo-joints obtained by swapping a pair of unresolved positions. This circulation is zero under compatible conditionals, and global order gaps decompose into sums of local circulations along adjacent swaps. We further separate incompatibility-driven path dependence from conditional-dependence error in parallel updates and from order-specific estimation error. The resulting framework provides inference-only diagnostics for testing when DLM decoding is genuinely order-free.
\end{abstract}

\section{Introduction}

Autoregressive language models factorize text generation into a strict left-to-right product of conditional probabilities. This factorization yields clear probabilistic semantics, but it also imposes an unavoidable serial bottleneck: the $i$th token cannot be generated before the preceding $i-1$ tokens have been determined. Diffusion language models seek to bypass this limitation. Unlike autoregressive models, masked or discrete diffusion language models start from highly corrupted text states and gradually recover tokens through iterative denoising. Because denoising steps need not follow a fixed absolute position order, DLMs support arbitrary-order, block-wise, and fully parallel generation at the interface level \citep{austin2021d3pm,li2022diffusionlm,gulrajani2023plaid,sahoo2024mdlm,lou2024sedd,nie2025llada}.

This apparent freedom is precisely why DLMs are often seen as promising candidates for breaking the latency bottleneck of autoregressive inference. In practice, however, a central contradiction is emerging: while the model interface permits arbitrary order, the learned behavior is often far from order-invariant. Many fast DLMs degrade substantially under high decoding parallelism; to preserve quality, effective decoding policies frequently drift back toward left-to-right, confidence-based, or other highly structured unmasking schedules \citep{kim2025ordering,kim2025klass}. More directly, recent concurrent analyses have reported that practical DLMs can revert toward autoregressive-like decoding under stronger parallelism, and that masked diffusion models can exhibit locality and mask-induced context-comprehension failures that further complicate genuinely order-free generation \citep{li2026parallel,piskorz2026masks}.

Existing explanations usually operate at the level of data and tasks. Natural language is inherently sequential, and chain-of-thought supervision further reinforces the pattern that later steps depend on earlier ones. As a result, even a non-autoregressive training objective may lead a DLM to internalize autoregressive-like computation \citep{kim2025ordering,li2026parallel,piskorz2026masks}. These explanations capture important regularities, but they leave the key structural question under-specified. For a fixed reverse-time step, a DLM provides a family of local conditional distributions over unresolved coordinates. The central issue is whether those local conditionals are compatible with a common joint distribution or define only local rules for the decoder.

This perspective separates two levels of analysis. At the model level, different denoising orders induce different pseudo-joints $Q_{\theta,t}^{\pi}$, and local curl measures the log-ratio between nearby orderings. At the decoding level, operator commutators ask whether a particular update rule---argmax, sampling, thresholding, or another scheduler-specific operator---turns that pseudo-joint order gap into path sensitivity of the predictive state. Even with small curl, fully independent parallel updates can still fail when unresolved coordinates have large conditional total correlation. In language tasks, left-to-right order often re-emerges because it both stabilizes context and moves the model onto lower-error conditional contexts.

The paper makes four contributions:
\begin{itemize}
\item \textbf{Compatibility formulation.} We formulate arbitrary-order denoising as a pseudo-joint invariance problem.
\item \textbf{Local circulation diagnostic.} Elementary curl is exactly a swapped pseudo-joint log-ratio, and its expectation under one order is the corresponding order-swap KL divergence.
\item \textbf{Path-dependence characterization.} Exact order consistency on a fixed block holds if and only if every reachable elementary square is curl-free, and any global order gap decomposes along adjacent swaps.
\item \textbf{Failure-mode separation.} Bayes-optimal uniform masking yields zero curl, parallel failure also depends on conditional total correlation, and order preference is governed by order-specific conditional estimation error.
\end{itemize}

The paper is formal and mechanism-oriented. It establishes exact structural facts about pseudo-joints, identifies the operator-dependent route from incompatibility to path sensitivity, separates curl from conditional total correlation, and specifies empirical tests of the compatibility view.

\section{Related Work}

The relationship between this paper and existing literature can be summarized as ``DLM order sensitivity viewed through conditional compatibility.'' Foundational work on diffusion for language and discrete data includes D3PMs \citep{austin2021d3pm}, continuous-text diffusion in Diffusion-LM \citep{li2022diffusionlm}, likelihood-oriented large-scale diffusion language modeling in Plaid \citep{gulrajani2023plaid}, masked discrete diffusion in MDLM \citep{sahoo2024mdlm}, score-entropy discrete diffusion in SEDD \citep{lou2024sedd}, and recent scaling results for LLaDA \citep{nie2025llada}. Together, these papers establish that diffusion-style objectives can be extended to text generation in several distinct ways, and in some regimes can become competitive with strong autoregressive baselines.

A first background line comes from conditionally specified distributions, Markov random fields, Gibbs-style constructions, and dependency networks \citep{besag1974spatial,geman1984stochastic,arnoldpress1989compatible,arnold2001conditional,heckerman2000dependency,hobert1998functional}. These works ask when a family of local conditionals corresponds to a single joint distribution and what happens when compatibility fails. Our elementary-circulation diagnostic is intended as a local, time-dependent witness to that same underlying question for denoising conditionals. A closely related modern empirical line is \citet{wagner2024contests}, who test whether different completion or conditioning orders assign consistent span probabilities in language models. Their results provide a useful consistency-testing precedent for the methodological point that local predictive scores should not automatically be assumed to arise from one order-invariant global probability object.

A second background line comes from orderless or permutation-trained language modeling. NADE and its orderless extensions make variable order explicit inside tractable autoregressive density estimators, while XLNet trains across permutations of factorization order \citep{larochelle2011nade,uria2014deepnade,yang2019xlnet}. These models differ from DLMs in that their semantics are still built around exact factorization objectives. Our focus is on diffusion denoisers whose interface permits arbitrary update order even when the induced local conditionals are not jointly compatible.

Recent work has also started to analyze inference order and decoding strategy more directly. In particular, \citet{kim2025ordering} show both theoretically and empirically that token ordering is a central variable in masked diffusion inference, while \citet{kim2025klass} demonstrate that adaptive fast samplers can materially change the quality--speed trade-off. Concurrent analyses have further highlighted difficulties with truly parallel decoding and broader context-comprehension issues in masked diffusion models \citep{li2026parallel,piskorz2026masks}. We place these phenomena in a common framework of order-induced pseudo-joints: learned incompatibility yields order gaps, conditional dependence yields independent-parallel error even at zero curl, and the scheduler selects contexts on which local conditional estimation error is smaller. The result is a unified diagnostic and mechanistic account of these phenomena.

\section{DLMs as Time-Dependent Conditional Specifications}

\subsection{Local Conditionals Rather Than a Joint}

Consider a token sequence $x=(x_1,\dots,x_L)\in \V^L$ of length $L$, where each $x_i$ belongs to a vocabulary $\V$. In masked diffusion or discrete diffusion language modeling, the forward process gradually replaces original tokens with masks, noisy tokens, or other corrupted states; the reverse process starts from a highly corrupted state and gradually reconstructs fluent text. Let $z_t$ denote the partially observed state at time step $t$, and let $M_t$ denote the set of positions that remain masked or corrupted. A typical denoising predictor is then written as \citep{austin2021d3pm,gulrajani2023plaid,sahoo2024mdlm,lou2024sedd,nie2025llada}
\[
p_\theta(x_i \mid z_t, t), \qquad i\in M_t.
\]

For theoretical analysis, however, it is useful to freeze the observed part of $z_t$ and focus on the currently unresolved coordinates. Let $x=(x_1,\dots,x_m)\in \V^m$ denote an assignment on those active coordinates, with already observed tokens absorbed into the conditioning context. At time $t$, the denoiser therefore induces a family of local conditional specifications
\[
q_{\theta,t}(x_i \mid x_{-i}), \qquad i=1,\dots,m.
\]
This places DLM denoisers closer to the literature on conditionally specified models, Gibbs/MRF constructions, and dependency networks than to a jointly normalized left-to-right factorization \citep{besag1974spatial,geman1984stochastic,arnoldpress1989compatible,arnold2001conditional,heckerman2000dependency,hobert1998functional}. The model supplies a collection of local conditionals; the existence of a single global distribution that realizes them is a separate structural property.

The sampling interface differs sharply from autoregressive models: the model can update tokens outside a fixed left-to-right order. Single-position updates require an unmasking order; multi-position updates yield block-wise or parallel decoding. This freedom makes DLMs appealing, but arbitrary-order denoising requires compatible local conditionals across orders. When different orders induce different generation distributions, the scheduler becomes part of the effective generative rule.

\subsection{Order-Induced Pseudo-Joint Distributions}

Fix an observed context index set $S$ and a block $B\subseteq [L]\setminus S$ of unresolved positions. For any permutation $\pi=(\pi_1,\dots,\pi_{|B|})$ of $B$, the denoiser induces the order-specific pseudo-joint
\[
Q_{\theta,t}^\pi(x_B\mid x_S)
=
\prod_{m=1}^{|B|}
q_{\theta,t}(x_{\pi_m}\mid x_S,x_{\pi_{<m}}).
\]
We call $Q_{\theta,t}^\pi$ a pseudo-joint because it is the sequential product induced by one factorization path. A DLM supports arbitrary-order denoising on a block when these pseudo-joints agree across relevant orders:
\[
Q_{\theta,t}^\pi(x_B\mid x_S)
\approx
Q_{\theta,t}^{\pi'}(x_B\mid x_S).
\]
Order-free generation is thus a claim about pseudo-joint invariance across factorizations.

\subsection{Compatibility and Global Potentials}

The classical compatibility question is whether there exists a joint distribution
\[
\begin{aligned}
&\exists p_{\theta,t}(x)\\
&\text{s.t. } q_{\theta,t}(x_i\mid x_{-i}) = p_{\theta,t}(x_i\mid x_{-i})
\quad \text{for all } i.
\end{aligned}
\]
If such a joint exists, define the global potential
\[
\Phi_\theta(x,t)=\log p_{\theta,t}(x).
\]
This $\Phi_\theta$ serves as an idealized reference object for order-free denoising.

We define the potential through shift-invariant log conditional probability ratios rather than raw logits. If $\ell_i(a\mid x_{-i},t)$ denotes a pre-softmax logit, then the shift $\ell_i(a\mid x_{-i},t)\mapsto \ell_i(a\mid x_{-i},t)+c_i(x_{-i},t)$ leaves the conditional distribution unchanged. Raw logit levels are identifiable only up to a position-wise additive constant, so compatibility and circulation live at the level of log conditional ratios.

Compatible local conditionals make every order-induced pseudo-joint $Q_{\theta,t}^\pi(x_B\mid x_S)$ coincide with the same conditional distribution $p_{\theta,t}(x_B\mid x_S)$, regardless of $\pi$. Incompatible conditionals make order sensitivity structural: the scheduler becomes part of the effective generative rule.

This issue is concrete. Recent consistency-testing work on span probabilities in language models studies whether interchangeable completion or conditioning orders agree on the same span probability object \citep{wagner2024contests}. Those results support the same methodological point here: a global potential interpretation requires order-swapped local predictions to compose into a consistent probability object.

\begin{remark}[Exact compatibility is a restrictive special case]
Exact compatibility imposes a coupled family of equalities across coordinates, conditioning contexts, and reverse-time steps. In flexible neural denoisers trained only through local denoising losses, there is typically no architectural reason for these identities to hold exactly. We therefore treat zero elementary circulation as a meaningful and important special case, but not as the generic default.
\end{remark}

\section{Pseudo-Joints, Curl, and Path Dependence}

\subsection{Local Order Curl as a Pseudo-Joint Log-Ratio}

Fix a partial context $x_S$, two unresolved positions $i,j\notin S$, and candidate tokens $a,b\in\V$. Define the two local order-induced pseudo-joints
\[\begin{aligned}
Q_{\theta,t}^{i\to j}(a,b\mid x_S)
={}
q_{\theta,t}(a\mid x_S)\,q_{\theta,t}(b\mid x_S,x_i{=}a),\\
Q_{\theta,t}^{j\to i}(a,b\mid x_S)
={}
q_{\theta,t}(b\mid x_S)\,q_{\theta,t}(a\mid x_S,x_j{=}b).
\end{aligned}
\]
Define the local order curl
\[
\begin{aligned}
\Curl_{ij}^{a,b}(x_S,t)
=&\;\log q_{\theta,t}(x_i=a\mid x_S)\\
&+\log q_{\theta,t}(x_j=b\mid x_S,x_i=a)\\
&-\log q_{\theta,t}(x_j=b\mid x_S)\\
&-\log q_{\theta,t}(x_i=a\mid x_S,x_j=b).
\end{aligned}
\]

\begin{theorem}[Local curl equals a pseudo-joint log-ratio]
For every $x_S,t,i,j,a,b$,
\[
\Curl_{ij}^{a,b}(x_S,t)
=
\log \frac{Q_{\theta,t}^{i\to j}(a,b\mid x_S)}{Q_{\theta,t}^{j\to i}(a,b\mid x_S)}.
\]
Consequently, $\Curl_{ij}^{a,b}(x_S,t)=0$ for all $a,b$ if and only if
\[
Q_{\theta,t}^{i\to j}(\cdot,\cdot\mid x_S)
=
Q_{\theta,t}^{j\to i}(\cdot,\cdot\mid x_S)
\]
as distributions. Moreover,
\[
\begin{aligned}
&\E_{(a,b)\sim Q_{\theta,t}^{i\to j}(\cdot,\cdot\mid x_S)}
\bigl[\Curl_{ij}^{a,b}(x_S,t)\bigr]\\
&\qquad=
\KL\bigl(
Q_{\theta,t}^{i\to j}(\cdot,\cdot\mid x_S)
\,\|\,
Q_{\theta,t}^{j\to i}(\cdot,\cdot\mid x_S)\bigr).
\end{aligned}
\]
\end{theorem}

\noindent\textbf{Proof idea.} Expanding the two pseudo-joints gives the displayed log-ratio, and averaging that identity under $Q_{\theta,t}^{i\to j}$ yields the KL term; Appendix~A gives the full proof.

This is the key interpretation: curl is the log-density ratio between two denoising orders on the same local block. Because it is built entirely from conditional log-probabilities, it is automatically invariant to additive logit shifts. At the sample level we may aggregate these quantities as
\[
\mathrm{ECircAbs}(x_S,t)=\E_{i,j,a,b}\bigl[\lvert\Curl_{ij}^{a,b}(x_S,t)\rvert\bigr],
\]
and for cross-time or cross-model comparison one may normalize by the total magnitude of the four participating log-probability terms:
\[
\begin{aligned}
\widetilde{\Curl}_{ij}^{a,b}(x_S,t)
&=\\
&\frac{\lvert\Curl_{ij}^{a,b}(x_S,t)\rvert}{
\begin{aligned}[t]
&\lvert \log q_{\theta,t}(x_i=a\mid x_S)\rvert\\
&+\lvert \log q_{\theta,t}(x_j=b\mid x_S,x_i=a)\rvert\\
&+\lvert \log q_{\theta,t}(x_j=b\mid x_S)\rvert\\
&+\lvert \log q_{\theta,t}(x_i=a\mid x_S,x_j=b)\rvert
+\epsilon
\end{aligned}}.
\end{aligned}
\]
for a small $\epsilon>0$.

\subsection{Adjacent Swaps and Block-Level Path Dependence}

For a block $B$ and permutation $\pi$, recall the pseudo-joint
\[
Q_{\theta,t}^{\pi}(x_B\mid x_S)
=
\prod_{m=1}^{|B|} q_{\theta,t}(x_{\pi_m}\mid x_S,x_{\pi_{<m}}).
\]
Any two permutations can be connected by a sequence of adjacent swaps, but the local curl terms in the decomposition must be indexed by a specific swap path because the prefix context at each swap depends on the intermediate ordering.

\begin{theorem}[Path-specific adjacent-swap decomposition]
Let $\pi$ and $\pi'$ be two permutations of $B$, and fix an adjacent-swap path
\[
\Gamma:\qquad
\rho^{(0)}=\pi
\to \rho^{(1)}
\to \cdots
\to \rho^{(R)}=\pi'.
\]
Suppose the $r$th step swaps adjacent positions $k_r$ and $k_r+1$ in the current ordering $\rho^{(r-1)}$. Write
\[
i_r=\rho^{(r-1)}_{k_r},
\qquad
j_r=\rho^{(r-1)}_{k_r+1},
\]
and let
\[
S_r
=
S\cup
\{\rho^{(r-1)}_1,\dots,\rho^{(r-1)}_{k_r-1}\}
\]
be the prefix index set before the swapped adjacent pair. Then for every assignment $x_B$,
\[
\log \frac{Q_{\theta,t}^{\pi}(x_B\mid x_S)}{Q_{\theta,t}^{\pi'}(x_B\mid x_S)}
=
\sum_{r=1}^R
\Curl_{i_rj_r}^{x_{i_r},x_{j_r}}(x_{S_r},t).
\]
\end{theorem}

\noindent\textbf{Proof.} See Appendix~A.

This theorem makes the path-dependence statement exact: path dependence is accumulated local curl along a chosen adjacent-swap path. The individual contexts $S_r$ depend on the path $\Gamma$, but the total sum does not, because it telescopes to the fixed endpoint quantity $\log Q_{\theta,t}^{\pi}-\log Q_{\theta,t}^{\pi'}$. The relevant ``circulation'' is therefore not an endpoint prediction difference at the same state, but the log-ratio between pseudo-joints induced by different factorization paths.

The next proposition upgrades this pathwise identity into an exact characterization of order consistency on a fixed block.

\begin{proposition}[Exact order consistency equals reachable curl-freeness]
Fix $x_S$, a block $B\subseteq [L]\setminus S$, and a time $t$. Assume every local conditional appearing in the relevant pseudo-joints is strictly positive. Then the following are equivalent:
\begin{enumerate}
\item For every pair of permutations $\pi,\pi'$ of $B$ and every assignment $x_B$,
\[
Q_{\theta,t}^{\pi}(x_B\mid x_S)=Q_{\theta,t}^{\pi'}(x_B\mid x_S).
\]
\item For every reachable context index set $C=S\cup A$ with $A\subseteq B$, every assignment $x_C$ extending $x_S$, every distinct $i,j\in B\setminus A$, and every $a,b\in\V$,
\[
\Curl_{ij}^{a,b}(x_C,t)=0.
\]
\end{enumerate}
Equivalently, exact arbitrary-order semantics on $B$ relative to $x_S$ are the same as discrete curl-freeness on every reachable elementary square inside $B$.
\end{proposition}

\noindent\textbf{Proof.} See Appendix~A.

This proposition makes local curl an exact semantic criterion. On a fixed block, a DLM is genuinely arbitrary-order if and only if every reachable elementary square is curl-free. Any non-zero curl on a reachable square obstructs exact arbitrary-order consistency on that block.

\subsection{Pseudo-Joint Gaps Versus Decoder Commutators}

The pseudo-joint identities above are model-level statements about the family of conditionals induced by the denoiser. A practical decoder adds another layer of path sensitivity. Let $P_\theta(\cdot\mid z,t)$ denote the predictive object over unresolved coordinates at state $z$ and time $t$, and let $U_i$ denote a local update rule at position $i$. Then the operator commutator
\begin{equation}
\begin{aligned}
\mu_{ij}^{U}(z,t)
&\coloneqq
\mathbb{P}_{\theta}(\cdot \mid U_i U_j z,t), \\
\operatorname{Comm}_{ij}^{(U,D)}(z,t)
&\coloneqq
D\!\left(
\mu_{ij}^{U}(z,t),
\mu_{ji}^{U}(z,t)
\right).
\end{aligned}
\label{eq:operator-commutator}
\end{equation}
measures whether a chosen rule turns a pseudo-joint order gap into a discrepancy between downstream predictive objects. High curl does not force every decoder to diverge, but commit-style operators such as argmax writing, thresholded updates, and finite-support sampling are natural mechanisms by which pseudo-joint mismatch becomes decoding-time path sensitivity.

\section{Bayes Limits, Parallel Independence Error, and Order Preference}

The pseudo-joint view clarifies three further questions: when curl must vanish, why parallel decoding can still fail even when curl is zero, and how a model can acquire a preferred order without any theorem saying that language itself generates non-zero curl.

\begin{theorem}[Bayes-optimal uniform masking implies zero curl]
Suppose the denoiser is trained under a uniform masking objective and, in the infinite-capacity Bayes-optimal limit, satisfies
\[
q_{\theta,t}(x_i\mid x_S)=p_{\mathrm{data}}(x_i\mid x_S)
\]
for every visible context $x_S$ and unresolved coordinate $i$. Then for every $x_S,t,i,j,a,b$,
\[
\begin{aligned}
Q_{\theta,t}^{i\to j}(a,b\mid x_S)
&=p_{\mathrm{data}}(a,b\mid x_S)\\
&=Q_{\theta,t}^{j\to i}(a,b\mid x_S),
\end{aligned}
\]
and therefore
\[
\Curl_{ij}^{a,b}(x_S,t)=0.
\]
\end{theorem}

\noindent\textbf{Proof idea.} Under Bayes optimality, both factorization orders equal the same data conditional $p_{\mathrm{data}}(a,b\mid x_S)$ by the chain rule, so the curl vanishes; Appendix~A gives the full proof.

This theorem prevents an overly strong interpretation of the framework. Data directionality alone does not imply non-zero compatibility curl. Non-zero curl is a property of the learned denoiser away from Bayes-optimal compatibility, arising from finite capacity, incomplete coverage of masking contexts, imperfect optimization, calibration mismatch, or decoding-time approximation.

\begin{theorem}[Independent parallel updates incur conditional total correlation]
For a block $B$, define the one-shot independent parallel approximation
\[
Q_{\theta,t}^{\mathrm{par}}(x_B\mid x_S)=\prod_{i\in B} q_{\theta,t}(x_i\mid x_S).
\]
If the denoiser is Bayes optimal so that $q_{\theta,t}=p$, then
\[
\begin{aligned}
&\KL\Bigl(p(x_B\mid x_S)
\,\Big\|\,
Q_{\theta,t}^{\mathrm{par}}(x_B\mid x_S)\Bigr)\\
&\quad=\KL\Bigl(p(x_B\mid x_S)
\,\Big\|\,
\prod_{i\in B} p(x_i\mid x_S)\Bigr)\\
&\quad=\TC(X_B\mid X_S),
\end{aligned}
\]
where
\[
\begin{aligned}
\TC(X_B\mid X_S)
=\sum_{i\in B}H(X_i\mid X_S)
-H(X_B\mid X_S).
\end{aligned}
\]
\end{theorem}

\noindent\textbf{Proof idea.} Substituting the Bayes-optimal condition $q=p$ reduces the KL term to the conditional total-correlation definition, and the entropy form follows by expansion; Appendix~A gives the full proof.

Parallel degradation has multiple sources. Even with zero compatibility curl, one-shot independent updates can fail whenever unresolved coordinates remain strongly conditionally dependent. A useful conceptual decomposition is therefore:
\[
\begin{aligned}
&\text{parallel degradation}\\
&\qquad\approx \text{conditional dependence / }\TC\\
&\qquad\phantom{\approx}+\text{ model incompatibility / curl}\\
&\qquad\phantom{\approx}+\text{ estimation and calibration error}.
\end{aligned}
\]
This display should be read as a structural decomposition of failure modes, not as a claim of an exact additive theorem.

\begin{theorem}[Order preference is governed by order-specific conditional estimation error]
For any order $\pi$,
\[
\begin{aligned}
&\E_p\bigl[-\log Q_{\theta,t}^{\pi}(X_B\mid X_S)\bigr]\\
&=H_p(X_B\mid X_S)
+\KL\bigl(p(X_B\mid X_S)
\,\|\,
Q_{\theta,t}^{\pi}(X_B\mid X_S)\bigr).
\end{aligned}
\]
Moreover, the KL term decomposes as
\[
\begin{aligned}
&\KL\bigl(p(X_B\mid X_S)
\,\|\,
Q_{\theta,t}^{\pi}(X_B\mid X_S)\bigr)\\
&\qquad=
\sum_{m=1}^{|B|}
\E_p\Bigl[
\begin{aligned}[t]
\KL\bigl(
p(X_{\pi_m}\mid X_S,X_{\pi_{<m}})
\,\|\, q_{\theta,t}(X_{\pi_m}\mid X_S,X_{\pi_{<m}})
\bigr)
\end{aligned}
\Bigr].
\end{aligned}
\]
\end{theorem}

\noindent\textbf{Proof.} See Appendix~A.

Define the order-specific local estimation error
\[
\epsilon_i(C)=\KL\bigl(p(X_i\mid C)\,\|\,q_{\theta,t}(X_i\mid C)\bigr).
\]
Then the best order is the one whose induced context distribution makes the cumulative expected local errors smallest. In particular, if prefix-like contexts are easier for the model than arbitrary mixed-mask contexts, the preferred decoding path can naturally become left-to-right even without any theorem saying that language itself forces signed curl.

\section{Mechanism: Why DLMs Drift Toward AR-Like Decoding}

The preceding theorems suggest a more careful mechanism than ``language is directional, therefore curl is non-zero.'' First, Theorem~3 shows that at Bayes-optimal uniform masking the local curl vanishes. Data directionality alone is therefore insufficient. Second, real denoisers operate under finite capacity, imperfect optimization, incomplete coverage of rare masking patterns, and approximate decoding. These factors generate pseudo-joint mismatch and non-zero local curl. Third, once some mismatch exists, Theorem~5 shows that order preference is governed by the contexts on which the model's local conditional error is smallest. If prefix-like contexts are easier to estimate than arbitrary mixed-mask contexts, a left-to-right schedule becomes a natural low-error path. Fourth, Theorem~4 shows that even near-zero curl does not make fully parallel decoding safe when unresolved coordinates retain large conditional total correlation.

Autoregressive-like collapse arises from the interaction of three ingredients: learned pseudo-joint mismatch, residual within-block dependence, and order-specific conditional estimation error. Better parallel DLMs must address all three.

\section{Validation Interface}

The framework yields a concrete validation interface. The empirical task is to test separate predictions about pseudo-joint order gaps, conditional dependence inside the block, and order-specific estimation error.

\paragraph{Order-gap prediction.}
For a fixed model and task, larger absolute local curl or larger expected two-order KL should predict stronger disagreement between order-induced pseudo-joints and greater fragility when the decoding order is perturbed.

\paragraph{TC versus curl.}
Even when pairwise curl is small, blocks with high conditional total correlation should remain resistant to one-shot independent parallel updates. This predicts a failure regime that is not explained by incompatibility alone.

\paragraph{Order-specific error profiling.}
If a scheduler performs substantially better than another, the advantage should correlate with the contexts it induces for subsequent predictions. Strong schedules should move the model onto lower-error conditional contexts, not just lower-entropy positions.

\paragraph{Trajectory diversification.}
If training exposes the model to multiple equivalent recovery trajectories or explicitly enforces cross-order agreement, then pseudo-joint order gaps should decrease and order-specific estimation asymmetry should weaken.

\paragraph{Synthetic control tasks.}
Synthetic data should allow separate manipulation of within-block dependence and learned incompatibility. Tasks with the same conditional dependence but different training coverage or architectural constraints should separate TC-driven failure from curl-driven failure.

\section{Operationalization}

The validation interface can be instantiated without retraining large models. The essential requirement is access to token conditionals or logits at fixed reverse-time states, so that pseudo-joint ratios, local estimation errors, and post-update predictive objects can be queried directly.

\subsection{Model Access and Scope}

A practical starting point is to use publicly accessible DLMs or masked diffusion models for which local conditional probabilities or logits can be queried directly, such as Plaid-style likelihood-trained diffusion models \citep{gulrajani2023plaid}, MDLM \citep{sahoo2024mdlm}, SEDD \citep{lou2024sedd}, and LLaDA \citep{nie2025llada}. If compute is limited, an initial study can focus on smaller checkpoints and inference-only diagnostics before scaling up the same measurements. For reproducibility, the implementation should report checkpoint identifiers, parameter scales, tokenizer details, reverse-time schedule, and hardware.

\subsection{Pseudo-Joint Gap Estimation}

For each sample at a chosen time step $t$, select active coordinate pairs $(i,j)$ and candidate token pairs $(a,b)$. Compute the local curl
\[
\Curl_{ij}^{a,b}(x_S,t)
=\log \frac{Q_{\theta,t}^{i\to j}(a,b\mid x_S)}{Q_{\theta,t}^{j\to i}(a,b\mid x_S)}
\]
and summarize both absolute and normalized versions. A useful companion statistic is the empirical two-order KL from Theorem~1, estimated by averaging curl under samples drawn from one order-induced pseudo-joint.

\subsection{Commutator Diagnostics}

To stay close to actual decoding, explicitly run two local update paths under a chosen rule $U$:
\[
z_{ij}=U_j(U_i(z)),\qquad z_{ji}=U_i(U_j(z)).
\]
Then compare the downstream predictive objects:
\[
\begin{aligned}
\Comm_{ij}^{(U,\sqrt{\JS})}(z,t)
=\Bigl[\JS\bigl(
P_\theta(\cdot\mid z_{ij},t),
P_\theta(\cdot\mid z_{ji},t)
\bigr)\Bigr]^{1/2}.
\end{aligned}
\]
This quantifies how much the pseudo-joint order gap survives the decoder and appears as prediction-level path sensitivity.

\subsection{Conditional-Dependence Controls}

For any candidate block $B$, estimate the independent-parallel gap
\[
\KL\Bigl(\widehat p(x_B\mid x_S)
\,\Big\|\,
\prod_{i\in B}\widehat p(x_i\mid x_S)\Bigr)
\]
or a tractable proxy such as the sum of pairwise conditional mutual informations inside $B$. These quantities serve as empirical controls for the TC term from Theorem~4 and prevent parallel-failure analyses from over-crediting curl.

\subsection{Order-Specific Error Profiling}

For each candidate order $\pi$, evaluate the cumulative conditional cross-entropy
\[
\sum_{m=1}^{|B|}
-\log q_{\theta,t}(x_{\pi_m}\mid x_S,x_{\pi_{<m}})
\]
on held-out data, and compare it with the same quantity under alternative orders. This directly measures the order-specific estimation-error profile from Theorem~5. In practice, one can stratify these losses by context type---prefix-like, random-mask, high-entropy, or synthetically controlled---to see where the model most strongly prefers a particular factorization path.

\subsection{Parallelism Stress Test}

For the same model and task, compare several decoding regimes: sequential recovery, small-block recovery, large-block recovery, and fully parallel or few-step recovery. Regress the performance drop under increased parallelism against three families of predictors: pseudo-joint order-gap statistics, TC proxies, and order-specific conditional error. This directly tests whether the framework explains failure better than entropy or confidence alone.

\subsection{Scheduler Analysis}

Compare left-to-right, random, confidence-based, and learned schedulers in terms of the blocks and contexts they induce. Strong schedulers should avoid both high-curl swaps and high-TC blocks, while also moving the model toward contexts on which its conditional error is lower. On this view, scheduler quality comes from navigating all three terms in the mechanism, not from confidence alone.

\subsection{Synthetic Dependency Tasks}

Construct synthetic settings that separately tune conditional dependence and compatibility. For example, one can hold the data-generating dependency graph fixed while varying training coverage over masking patterns, or hold masking coverage fixed while changing the data dependency graph from chain-like to exchangeable. Measuring curl, TC proxies, and order-specific error in these settings provides a clean causal test of the framework.

\section{Potential Regularization and Improvement Directions}

Although this paper is primarily analytical, the framework naturally suggests several improvement directions. The target is reduced incompatibility on coordinates that the system intends to update jointly or in arbitrary order, while preserving directional asymmetries that are useful for modeling language.

\subsection{Elementary-Circulation Regularization}

During training, one can sample position pairs and candidate token assignments and explicitly minimize a normalized compatibility penalty such as
\[
\mathcal{L}_{\mathrm{ecirc}}
=\E_{x_S,t,i,j,a,b}\bigl[(\widetilde{\Curl}_{ij}^{a,b}(x_S,t))^2\bigr].
\]
Applied selectively to candidate parallel blocks, this regularizer encourages nearby orders to induce more consistent pseudo-joints where parallel decoding is desired. Its main cost is additional computation, and overly aggressive regularization could suppress useful asymmetries, so the regularizer should target candidate parallel blocks.

\subsection{Commutator-Aware Scheduling}

At decoding time, block selection should trade off confidence, commutator conflict, and within-block dependence. For a candidate block $B$, define
\[
\mathrm{Conflict}(B)=\sum_{i,j\in B,\,i<j} \Comm_{ij}^{(U,D)}(z,t).
\]
A practical scheduler can then trade off ``high confidence,'' ``low commutator conflict,'' and ``manageable conditional dependence,'' providing a more structured way to control parallelism than token-wise confidence alone.

\subsection{Parallel Trajectory Supervision}

If training data provide only a single reasoning path, such as a standard chain-of-thought trajectory, then the model can easily internalize a strong directional bias in where incompatibility is tolerated. A direct alternative is to construct multiple equivalent reasoning trajectories, locally permuted supervision signals, or block-independent auxiliary targets so that the model sees evidence that the same answer can be recovered through different paths. This should reduce path bias at the data level, especially on coordinates that the decoder is expected to update jointly.

\subsection{Potential-Based DLMs}

A more ambitious direction is to explicitly learn a joint model $p_{\theta,t}(x)$ together with its potential $\Phi_\theta(x,t)=\log p_{\theta,t}(x)$, and then derive local denoising decisions from compatible conditionals. If successful, such models could enforce stronger integrability by construction. The challenge is that joint normalization and scalable training remain difficult in large-vocabulary, discrete-state, multi-step generation settings.

\section{Conclusion}

This paper reframes arbitrary-order denoising in diffusion language models through order-induced pseudo-joint distributions. The core condition is no longer verbal: for a block $B$ and partial context $S$, different denoising orders should induce approximately the same pseudo-joint $Q_{\theta,t}^{\pi}(x_B\mid x_S)$. Local curl is then exactly the log-ratio between two swapped pseudo-joints, and its expectation under one order is the corresponding order-swap KL divergence. Proposition~1 sharpens this further: exact order consistency on a fixed block is equivalent to curl-freeness on every reachable elementary square.

This viewpoint yields three clarifications. First, path dependence is accumulated local curl along adjacent swaps. Second, Bayes-optimal uniform masking implies zero curl, so data directionality alone does not force incompatibility. Third, parallel degradation has multiple components: even at zero curl, one-shot independent parallel updates incur a conditional-total-correlation penalty, and order preference is governed by the contexts on which the denoiser's conditional estimation error is smallest.

From a research-program perspective, the paper offers a criterion for what it means for a language diffusion model to be genuinely parallel: the interface must permit arbitrary-order generation, the induced pseudo-joints must remain stable across orders, the block dependence must be manageable at the intended update granularity, and the model must support stable context distributions across schedules. Future work can therefore proceed along three directions: measure pseudo-joint order gaps at scale, separate TC-driven failure from incompatibility-driven failure, and design training objectives that reduce harmful order-specific estimation asymmetries.

\section{Limitations}

This paper offers a mechanism-oriented framework, a set of formal diagnostics, and a concrete validation interface, but it does not yet include the broad empirical evidence required to establish the scope of the claim across many diffusion language model families. In particular, the paper does not yet show whether pseudo-joint-gap, TC, and order-specific-error measurements consistently outperform simpler proxies such as token entropy, confidence, or token distance when predicting failure under parallel decoding. The mechanism claim should therefore be read as a testable structural hypothesis rather than a closed empirical account.

A second limitation is that the key objects introduced here---local curl, order-induced pseudo-joints, operator commutators, and TC proxies---may be expensive to estimate for large models or long sequences. Their practical utility will depend on whether low-variance approximations can be computed efficiently enough to make the diagnostics usable at scale.

A third limitation is that the decomposition into incompatibility, conditional dependence, and estimation error is conceptually clean but not yet operationally unique. In practice these terms may interact, and empirical work will need careful controls to disentangle them.

A fourth limitation is conceptual. Low curl is not universally desirable, and neither is low conditional dependence. Language contains real structure. Our claim is therefore local and operational: arbitrary-order or fully parallel denoising becomes unsafe when the chosen update granularity outruns either compatibility or conditional-independence assumptions; it does not imply that diffusion language models are categorically inferior to autoregressive models \citep{nie2025llada}.

\bibliographystyle{plainnat}
\bibliography{refs}

\appendix

\section{Proofs and Technical Clarifications}

\subsection{Proof of Theorem~1}

\begin{proof}
By definition,
\[
Q_{\theta,t}^{i\to j}(a,b\mid x_S)
=
q_{\theta,t}(a\mid x_S)\,q_{\theta,t}(b\mid x_S,x_i{=}a)
\]
and
\[
Q_{\theta,t}^{j\to i}(a,b\mid x_S)
=
q_{\theta,t}(b\mid x_S)\,q_{\theta,t}(a\mid x_S,x_j{=}b).
\]
Taking the logarithm of the ratio gives
\[
\begin{aligned}
\log \frac{Q_{\theta,t}^{i\to j}(a,b\mid x_S)}{Q_{\theta,t}^{j\to i}(a,b\mid x_S)}
=&\;\log q_{\theta,t}(a\mid x_S)\\
&+\log q_{\theta,t}(b\mid x_S,x_i{=}a)\\
&-\log q_{\theta,t}(b\mid x_S)\\
&-\log q_{\theta,t}(a\mid x_S,x_j{=}b),
\end{aligned}
\]
which is exactly $\Curl_{ij}^{a,b}(x_S,t)$.

If $\Curl_{ij}^{a,b}(x_S,t)=0$ for all $a,b$, then the displayed ratio equals $1$ for all $a,b$, so
\[
Q_{\theta,t}^{i\to j}(a,b\mid x_S)=Q_{\theta,t}^{j\to i}(a,b\mid x_S)
\]
pointwise and hence the two distributions are equal. The converse is immediate from the same identity.

For the expectation claim,
\[
\begin{aligned}
&\E_{(a,b)\sim Q_{\theta,t}^{i\to j}(\cdot,\cdot\mid x_S)}
\bigl[\Curl_{ij}^{a,b}(x_S,t)\bigr]\\
&\qquad=
\sum_{a,b} Q_{\theta,t}^{i\to j}(a,b\mid x_S)
\log \frac{Q_{\theta,t}^{i\to j}(a,b\mid x_S)}{Q_{\theta,t}^{j\to i}(a,b\mid x_S)}\\
&\qquad=
\KL\bigl(
Q_{\theta,t}^{i\to j}(\cdot,\cdot\mid x_S)
\,\|\,
Q_{\theta,t}^{j\to i}(\cdot,\cdot\mid x_S)
\bigr),
\end{aligned}
\]
which is the definition of KL divergence.
\end{proof}

\subsection{Proof of Theorem~2}

\begin{proof}
For each step $r$, let
\[
A_r=(\rho^{(r-1)}_1,\dots,\rho^{(r-1)}_{k_r-1})
\qquad\text{and}\qquad
D_r=(\rho^{(r-1)}_{k_r+2},\dots,\rho^{(r-1)}_{|B|}),
\]
so that
\[
\rho^{(r-1)}=(A_r,i_r,j_r,D_r)
\qquad\text{and}\qquad
\rho^{(r)}=(A_r,j_r,i_r,D_r).
\]
Also write $C_r=S\cup A_r$, which is exactly the prefix index set $S_r$ in the theorem statement. Expanding the sequential products along the two adjacent orders gives
\[
\begin{aligned}
Q_{\theta,t}^{\rho^{(r-1)}}(x_B\mid x_S)
&=Q_{A_r}(x_{A_r}\mid x_S)\,q_{\theta,t}(x_{i_r}\mid x_{C_r})\\
&\quad\cdot q_{\theta,t}(x_{j_r}\mid x_{C_r},x_{i_r})\\
&\quad\cdot Q_{D_r}(x_{D_r}\mid x_{C_r},x_{i_r},x_{j_r}),
\end{aligned}
\]
and
\[
\begin{aligned}
Q_{\theta,t}^{\rho^{(r)}}(x_B\mid x_S)
&=Q_{A_r}(x_{A_r}\mid x_S)\,q_{\theta,t}(x_{j_r}\mid x_{C_r})\\
&\quad\cdot q_{\theta,t}(x_{i_r}\mid x_{C_r},x_{j_r})\\
&\quad\cdot Q_{D_r}(x_{D_r}\mid x_{C_r},x_{i_r},x_{j_r}).
\end{aligned}
\]
The prefix factor $Q_{A_r}$ is identical in both expressions, and the suffix factor $Q_{D_r}$ is also identical because after the adjacent swap both $i_r$ and $j_r$ have entered the conditioning context. Therefore,
\[
\begin{aligned}
&\log \frac{Q_{\theta,t}^{\rho^{(r-1)}}(x_B\mid x_S)}{Q_{\theta,t}^{\rho^{(r)}}(x_B\mid x_S)}
=\\
&\log q_{\theta,t}(x_{i_r}\mid x_{C_r})
+\log q_{\theta,t}(x_{j_r}\mid x_{C_r},x_{i_r})\\
&-\log q_{\theta,t}(x_{j_r}\mid x_{C_r})
-\log q_{\theta,t}(x_{i_r}\mid x_{C_r},x_{j_r})\\
&=
\Curl_{i_rj_r}^{x_{i_r},x_{j_r}}(x_{C_r},t).
\end{aligned}
\]
Summing these consecutive log-ratios along the path $\Gamma$ telescopes:
\[
\begin{aligned}
\log \frac{Q_{\theta,t}^{\pi}(x_B\mid x_S)}{Q_{\theta,t}^{\pi'}(x_B\mid x_S)}
&=
\sum_{r=1}^{R}
\log \frac{Q_{\theta,t}^{\rho^{(r-1)}}(x_B\mid x_S)}{Q_{\theta,t}^{\rho^{(r)}}(x_B\mid x_S)}\\
&=
\sum_{r=1}^{R}
\Curl_{i_rj_r}^{x_{i_r},x_{j_r}}(x_{S_r},t),
\end{aligned}
\]
which is the claimed decomposition.
\end{proof}

\subsection{Proof of Proposition~1}

\begin{proof}
Assume first that (ii) holds. Let $\pi$ and $\pi'$ be any two permutations of $B$, and let $\Gamma$ be any adjacent-swap path from $\pi$ to $\pi'$. Every swap along $\Gamma$ occurs at some reachable prefix context $C_r=S\cup A_r$ with $A_r\subseteq B$, so by (ii) each term in Theorem~2 is zero. Hence for every assignment $x_B$,
\[
\log \frac{Q_{\theta,t}^{\pi}(x_B\mid x_S)}{Q_{\theta,t}^{\pi'}(x_B\mid x_S)}=0.
\]
Because the relevant pseudo-joint factors are strictly positive, the logarithm is well defined and the last display implies
\[
Q_{\theta,t}^{\pi}(x_B\mid x_S)=Q_{\theta,t}^{\pi'}(x_B\mid x_S).
\]
This proves (i).

Conversely, assume (i). Fix a reachable context index set $C=S\cup A$ with $A\subseteq B$, an assignment $x_C$ extending $x_S$, distinct $i,j\in B\setminus A$, and tokens $a,b\in\V$. Let
\[
D=B\setminus (A\cup\{i,j\}),
\]
choose any permutation $\alpha$ of $A$, and choose any permutation $\delta$ of $D$. Consider the two full orders
\[
\pi=(\alpha,i,j,\delta)
\qquad\text{and}\qquad
\pi'=(\alpha,j,i,\delta).
\]
By (i), for every completion $x_D$ of the remaining coordinates,
\[
Q_{\theta,t}^{\pi}(x_A,a,b,x_D\mid x_S)
=
Q_{\theta,t}^{\pi'}(x_A,a,b,x_D\mid x_S).
\]
Expanding both sides, the prefix factors determined by $\alpha$ are identical, and the suffix factors determined by $\delta$ are also identical because in both orders they condition on the same realized variables $x_A$, $x_i=a$, and $x_j=b$. Cancelling these common factors yields
\[
q_{\theta,t}(a\mid x_C)
q_{\theta,t}(b\mid x_C,x_i{=}a)\\
=
q_{\theta,t}(b\mid x_C)
q_{\theta,t}(a\mid x_C,x_j{=}b).
\]
By Theorem~1, this equality of the two local pseudo-joints is equivalent to
\[
\Curl_{ij}^{a,b}(x_C,t)=0.
\]
Since $C$, $x_C$, $i$, $j$, $a$, and $b$ were arbitrary, (ii) follows.
\end{proof}

\subsection{Proof of Theorem~3}

\begin{proof}
Fix $x_S$, $t$, $i$, $j$, $a$, and $b$. By the Bayes-optimal assumption,
\[
q_{\theta,t}(a\mid x_S)=p_{\mathrm{data}}(a\mid x_S)
\]
and, because $x_S$ together with $x_i=a$ is also a visible context for coordinate $j$,
\[
q_{\theta,t}(b\mid x_S,x_i{=}a)=p_{\mathrm{data}}(b\mid x_S,x_i{=}a).
\]
Therefore,
\[
\begin{aligned}
Q_{\theta,t}^{i\to j}(a,b\mid x_S)
&=q_{\theta,t}(a\mid x_S)
q_{\theta,t}(b\mid x_S,x_i{=}a)\\
&=p_{\mathrm{data}}(a\mid x_S)
p_{\mathrm{data}}(b\mid x_S,x_i{=}a)\\
&=p_{\mathrm{data}}(a,b\mid x_S),
\end{aligned}
\]
where the last step is the chain rule. The same argument with $i$ and $j$ swapped gives
\[
\begin{aligned}
Q_{\theta,t}^{j\to i}(a,b\mid x_S)
&=q_{\theta,t}(b\mid x_S)
q_{\theta,t}(a\mid x_S,x_j{=}b)\\
&=p_{\mathrm{data}}(b\mid x_S)
p_{\mathrm{data}}(a\mid x_S,x_j{=}b)\\
&=p_{\mathrm{data}}(a,b\mid x_S).
\end{aligned}
\]
Hence
\[
Q_{\theta,t}^{i\to j}(a,b\mid x_S)
=p_{\mathrm{data}}(a,b\mid x_S)
=Q_{\theta,t}^{j\to i}(a,b\mid x_S).
\]
Applying Theorem~1 then gives $\Curl_{ij}^{a,b}(x_S,t)=0$.
\end{proof}

\subsection{Proof of Theorem~4}

\begin{proof}
Under the Bayes-optimal condition $q_{\theta,t}=p$, the independent parallel approximation becomes
\[
Q_{\theta,t}^{\mathrm{par}}(x_B\mid x_S)=\prod_{i\in B} p(x_i\mid x_S).
\]
Substituting this into the KL divergence yields
\[
\begin{aligned}
&\KL\Bigl(p(x_B\mid x_S)
\,\Big\|\,
Q_{\theta,t}^{\mathrm{par}}(x_B\mid x_S)\Bigr)\\
&\qquad=
\KL\Bigl(p(x_B\mid x_S)
\,\Big\|\,
\prod_{i\in B} p(x_i\mid x_S)\Bigr).
\end{aligned}
\]
By the definition of conditional total correlation, the right-hand side is $\TC(X_B\mid X_S)$.

For the entropy form, expand the KL divergence under the joint law of $(X_B,X_S)$:
\[
\begin{aligned}
\TC(X_B\mid X_S)
&=
\E_p\Biggl[
\log \frac{p(X_B\mid X_S)}{\prod_{i\in B} p(X_i\mid X_S)}
\Biggr]\\
&=
\E_p\bigl[\log p(X_B\mid X_S)\bigr]
-\sum_{i\in B}\E_p\bigl[\log p(X_i\mid X_S)\bigr]\\
&=
-H(X_B\mid X_S)+\sum_{i\in B} H(X_i\mid X_S),
\end{aligned}
\]
which is exactly
\[
\TC(X_B\mid X_S)=\sum_{i\in B}H(X_i\mid X_S)-H(X_B\mid X_S).
\]
\end{proof}

\subsection{Proof of Theorem~5}

\begin{proof}
For the first identity,
\[
\begin{aligned}
&\KL\bigl(p(X_B\mid X_S)
\,\|\,
Q_{\theta,t}^{\pi}(X_B\mid X_S)\bigr)\\
&\qquad=
\E_p\bigl[\log p(X_B\mid X_S)-\log Q_{\theta,t}^{\pi}(X_B\mid X_S)\bigr].
\end{aligned}
\]
Rearranging gives
\[
\begin{aligned}
&\E_p\bigl[-\log Q_{\theta,t}^{\pi}(X_B\mid X_S)\bigr]
=\\
&-\E_p\bigl[\log p(X_B\mid X_S)\bigr]
+\KL\bigl(p(X_B\mid X_S)
\,\|\,
Q_{\theta,t}^{\pi}(X_B\mid X_S)\bigr)\\
&=
H_p(X_B\mid X_S)
+\KL\bigl(p(X_B\mid X_S)
\,\|\,
Q_{\theta,t}^{\pi}(X_B\mid X_S)\bigr),
\end{aligned}
\]
which is the first claim.

For the second claim, apply the chain rule to both the target conditional and the pseudo-joint in the same order $\pi$:
\[
p(X_B\mid X_S)=\prod_{m=1}^{|B|} p(X_{\pi_m}\mid X_S,X_{\pi_{<m}})
\]
and
\[
Q_{\theta,t}^{\pi}(X_B\mid X_S)=\prod_{m=1}^{|B|} q_{\theta,t}(X_{\pi_m}\mid X_S,X_{\pi_{<m}}).
\]
Therefore,
\[
\begin{aligned}
&\KL\bigl(p(X_B\mid X_S)
\,\|\,
Q_{\theta,t}^{\pi}(X_B\mid X_S)\bigr)\\
&\qquad=
\E_p\Biggl[
\sum_{m=1}^{|B|}
\log \frac{p(X_{\pi_m}\mid X_S,X_{\pi_{<m}})}{q_{\theta,t}(X_{\pi_m}\mid X_S,X_{\pi_{<m}})}
\Biggr]\\
&\qquad=
\sum_{m=1}^{|B|}
\E_p\Biggl[
\log \frac{p(X_{\pi_m}\mid X_S,X_{\pi_{<m}})}{q_{\theta,t}(X_{\pi_m}\mid X_S,X_{\pi_{<m}})}
\Biggr].
\end{aligned}
\]
Now condition on $(X_S,X_{\pi_{<m}})$ and apply the definition of KL divergence inside each term:
\[
\begin{aligned}
&\E_p\Biggl[
\log \frac{p(X_{\pi_m}\mid X_S,X_{\pi_{<m}})}{q_{\theta,t}(X_{\pi_m}\mid X_S,X_{\pi_{<m}})}
\Biggr]\\
&\qquad=
\E_p\Bigl[
\KL\bigl(
p(X_{\pi_m}\mid X_S,X_{\pi_{<m}})
\,\|\,
q_{\theta,t}(X_{\pi_m}\mid X_S,X_{\pi_{<m}})
\bigr)
\Bigr].
\end{aligned}
\]
Summing over $m$ yields the claimed decomposition.
\end{proof}

\subsection{Why Pseudo-Joints Are the Right Starting Object}

A DLM denoiser natively provides local conditionals, not a jointly normalized density over an unresolved block. Once a decoding order is chosen, those local conditionals can always be multiplied into a sequential product $Q_{\theta,t}^{\pi}(x_B\mid x_S)$. This object is operationally unavoidable because it is exactly what a step-by-step decoder uses. The compatibility question is whether these pseudo-joints agree across orders.

\subsection{Why Logit-Shift Invariance Is the Identifiable Object}

If the model emits logits $\ell_i(a\mid x_S,t)$, then adding any scalar $c_i(x_S,t)$ to every logit at coordinate $i$ leaves the conditional distribution unchanged after softmax. Objects defined directly from raw logit levels are therefore not identifiable from the model's predictive semantics alone. By contrast, log-probability ratios and pseudo-joint log-ratios are invariant to such shifts, which is why the paper defines curl at that level.

\end{document}